\documentclass[letterpaper]{article} 
\usepackage{aaai20}  
\usepackage{times}  
\usepackage{helvet} 
\usepackage{courier}  
\usepackage[hyphens]{url}  
\usepackage{graphicx} 
\usepackage{graphicx}  
\def\year{2020}\relax

\usepackage{microtype}

\title{Convex Hull Monte-Carlo Tree-Search}
\author{
Michael Painter, Bruno Lacerda, Nick Hawes \\
Oxford Robotics Institute, University of Oxford \\
Oxford, UK\\
\{mpainter, bruno, nickh\}@robots.ox.ac.uk
}

\usepackage{amsmath,amssymb,amsthm,amsfonts}
\usepackage{bm}
\usepackage{algorithm,algorithmicx,algpseudocode}

\usepackage{etoolbox}

\usepackage{balance}

\newcommand{\citet}[1]{\citeauthor{#1} \shortcite{#1}}
\newcommand{\citep}{\cite}

\AtBeginEnvironment{align}{%
  \small%
}

\theoremstyle{definition}

\newtheorem{defn}{Definition}

\newtheorem{thrm}{Theorem}
\newtheorem{ex}{Example}

\newenvironment{proofoutline}
 {\proof[Proof Outline]}
 {\endproof}

\DeclareMathOperator*{\argmax}{arg\,max}
\renewcommand{\d}[1]{\ensuremath{\operatorname{d}\!{#1}}}

\newcommand{\cl}[1]{\mathcal{#1}}
\newcommand{\mtt}[1]{\mathtt{#1}}

\newcommand{\set}[1]{\hat{\mathbf{#1}}}

\begin{document}

\maketitle

\begin{abstract}
    This work investigates Monte-Carlo planning for agents in stochastic environments, with multiple objectives. We propose the \textit{Convex Hull Monte-Carlo Tree-Search} (CHMCTS) framework, which builds upon \textit{Trial Based Heuristic Tree Search} and \textit{Convex Hull Value Iteration} (CHVI), as a solution to multi-objective planning in large environments. Moreover, we consider how to pose the problem of approximating multi-objective planning solutions as a \textit{contextual multi-armed bandits} problem, giving a principled motivation for how to select actions from the view of \textit{contextual regret}. This leads us to the use of \emph{Contextual Zooming} for action selection, yielding  \textit{Zooming CHMCTS}. We evaluate our algorithm using the \textit{Generalised Deep Sea Treasure} environment, demonstrating that Zooming CHMCTS can achieve a sublinear contextual regret and scales better than CHVI on a given computational budget.
\end{abstract}

\section{Introduction}

In \textit{Multi-Objective Planning under uncertainty} (MOPU) an agent has to plan in a stochastic environment while trading off multiple objectives. Thus, we often assume the agent's environment is  known and modelled as a \textit{Multi-Objective Markov Decision Process} (MOMDP).
%
This problem is useful to consider as agents often encounter environments with uncertainty in their dynamics. Additionally, it is often useful for agents to be able to balance different objectives, for which the priorities are not known a priori or may change over time.
MOPU has been applied to tackle problems in several domains, such as having a robot  trade-off between battery consumption and performance in its primary task~\citep{lahijanian2018resource} or between timely achievement of a primary tasks whilst achieving as many soft goals as possible~\citep{lacerda_icaps17}; or when planning how to lay new electrical lines while trading installation and operational costs for network reliability~\citep{sahoo2012multi}.
The objective of MOPU is to compute a set of values, and associated policies, known as a \textit{Pareto Front}, that is optimal for some prioritisation over the objectives. Unfortunately, because there is not a single optimum value, as is the case in single-objective planning, there is an additional dimension of complexity to handle. In single-objective planning under uncertainty two techniques have recently lead to a significant improvement in the scalability of algorithms: \textit{Monte-Carlo Tree-Search} (MCTS)~\citep{kocsis2006bandit} and \textit{Value Function Approximation}~\citep{sutton2000policy}. Note that these techniques do not have to be mutually exclusive, as demonstrated in a system such as Alpha Go~\citep{silver2017mastering}.  However,  there has been relatively little work in adapting either of these techniques to the multi-objective setting~\citep{perez2015multiobjective,xu2017chebyshev,chen2019pareto,Mossalam2016,abels2018dynamic}.

An additional gap that needs to be addressed in the multi-objective setting is the need for principled approaches to evaluate the \emph{online performance of trial based planning algorithms} (such as MCTS). 
In this paper, we propose a regret based metric to do so.
We consider a sequence of trials, each with a known priority over the different objectives, such that the performance of the algorithm can be mapped to a scalar value that we wish to optimise.
This leads us to use the notion of \textit{Contextual Regret} as a measure of the online performance of a MOPU algorithm.

The main contributions of this work are: (i) applying the notion of contextual regret to \textit{multi-objective planning}, and justify that exploration policies that achieve low contextual regret explore the trade-offs between objectives appropriately, as opposed to other metrics proposed in the literature; (2) proposing \textit{Contextual Zooming for Trees}, that outperforms prior work on this metric.

\section{Related Work}

In the \textit{Multi-Objective Multi-Armed Bandit} (MOMAB) problem, an agent must pick one of $K$ arms round by round such that we optimise a multi-objective payoff vector. The MOMAB problem can be considered a special case of MOPU, as it can be mapped to a finite horizon MOMDP with a single state with $K$ actions, one for each arm. 
\citet{drugan2013designing} address the MOMAB problem by defining a set of arms, the \textit{Pareto Front}, that are all considered optimal, in the sense that the performance in one objective cannot be improved without degrading the performance for least one of the other objectives. They extend the well known UCB1~\citep{auer2002finite} algorithm to the MOMAB problem, using a multi-dimensional confidence interval rather than a one-dimensional confidence interval. Because the algorithm was adapted from UCB1, they refer to it as \textit{ParetoUCB1}.

\textit{Multi-objective sequential decision making} (which includes \textit{multi-objective planning}) extends ideas from single-objective sequential decision making algorithms, to handle a vector of rewards. In such problems, the solution to be computed is a \textit{Pareto Front} or a \textit{Convex Hull}, which is generally accepted to represent every possible trade off between objectives that could be made. For an in depth introduction to the field, see~\citep{roijers2013survey}.

Prior work in \textit{Multi-Objective Monte-Carlo Tree-Search} (MOMCTS) can be divided into two categories, those that maintain a single scalar value estimate at each node, which we will call \textit{Point-Based MOMCTS}, and those that maintain an estimate of the Pareto front at each node. 

In terms of point-based MOMCTS, \citet{wang2012multi} maintain a global Pareto front for the root node, and during trials, they select successor nodes based on how ``close'' the value estimate in the child nodes are to the global Pareto front. In \citep{chen2019pareto} the Pareto front for a node is formed from the value estimates of the children nodes, and actions are selected by running ParetoUCB over those points. 

In terms of  maintaining an estimate of the Pareto front at each node, \citet{perez2013online,perez2015multiobjective,perez2016multi} define an MOMCTS algorithm for games with deterministic transitions. The deterministic assumption simplifies the operations for updating the Pareto fronts. Our algorithm uses generalised versions of these updates, allowing for stochastic settings too. The rule for selecting successor nodes is adapted from the UCB1 algorithm~\citep{auer2002finite} over the \textit{hypervolumes} of child nodes. The hypervolume can be thought of as the area under the Pareto front. \citet{xu2017chebyshev} use the same algorithm that is presented by \citet{perez2013online}, however, they replace the use of the hypervolumes in UCB1 with the \textit{Chebychev scalarization function}.

All the works described above either assume a deterministic environment or do not maintain an estimated Pareto front at each search node. In this paper, we will show why maintaining an estimated Pareto front at each search node is required to fully explore the space of solutions, and introduce an approach that does so for \emph{stochastic} models.


\section{Preliminaries}

\subsection{DP-UCT}

\citet{keller2013trial} introduce the \textit{Trial-Based Heuristic Tree Search} (THTS) framework, that generalises trial-based planning algorithms for (single-objective) Markov Decision Processes (MDPs). In particular this framework can be specialised to give the \textit{Monte-Carlo Tree-Search} (MCTS) algorithm~\citep{browne2012survey} and the (UCT) algorithm~\citep{kocsis2006bandit}, the most used variant of MCTS, which uses \textit{UCB applied to trees} for action selection. THTS builds a search tree from \textit{decision nodes} and \textit{chance nodes} that correspond to state and state-action pairs in an MDP, respectively. Moreover, THTS is defined modularly, with different variants of algorithms being specified using seven functions, including  
\texttt{selectAction}, 
\texttt{backupDecisionNode} and 
\texttt{backupChanceNode}. For completeness we give an overview of THTS and pseudocode in Appendix A.

Due to the modularity of THTS, it is  easy to arrive at new algorithms by altering one or a few of the seven functions. \citet{keller2013trial} utilise this technique to arrive at DP-UCT, an adaption of the standard UCT algorithm that replaces Monte-Carlo backups with dynamic programming backups, and is shown empirically to outperform UCT in many domains.
Broadly speaking, DP-UCT algorithm is split into trials. Each trial traverses the tree until a leaf node is found. Once a leaf node has been reached, all nodes that were visited during the trial are backed up to update their value estimates. \texttt{selectAction} runs UCB1 at each decision node and \texttt{selectOutcome} samples successor states for each chance node. The \texttt{backupDecisionNode} and \texttt{backupChanceNode} perform Bellman backups in each node, using its children as the successor states in the backup.

In this work we, specify \textit{Convex Hull Monte-Carlo Tree-Search} under the THTS framework, by describing the adaptations made to  DP-UCT, rather than re-describing the standard parts of the algorithm. In particular, Section~\ref{sec:chmcts} details how to replace the \texttt{backupDecisionNode} and \texttt{backupChanceNode} functions in DP-UCT and Section~\ref{sec:action_selection} describes how to alter the \texttt{selectAction} function.

\subsection{Multi-Objective Planning under Uncertainty}

We will model MOPU problems using a Multi-Objective Markov Decision Process: 

\begin{defn} \label{def:momdp}
    A finite horizon \textit{Multi-Objective Markov Decision Process} (MOMDP) is a tuple $M=\langle S,A,\mathbf{R},T,\bar{s},H \rangle$, where:
    $S$ is a finite set of \textit{states}; 
    $A$ is a finite set of \textit{actions};
    $\mathbf{R}:S\times A \rightarrow \mathbb{R}^D$ is a \textit{D-dimensional reward function}, specifying $D$ immediate rewards for taking action $a$ when in state $s$;
    $T:S\times A\times S \rightarrow [0,1]$ is a \textit{transition function} specifying for each state, action and next state, the probability that the next state occurs given the current state and action;
    $\bar{s}$ is the \textit{initial state};
    and 
    $H\in\mathbb{N}$ is a finite \textit{horizon}. 
\end{defn}

In MOPU we are concerned with the optimisation of a vector $\mathbf{V}^\pi$, the sum of each reward observed over the $H$ time-steps, by an agent following a policy $\pi$. A (stochastic) policy $\pi$  maps  histories of visited states $\cl{H}$
to distributions over actions: given a history of visited states $h \in \cl{H}$, $\pi(h)(a)$ represents the probability of the policy choosing   to execute $a$, given that the history of visited states is $h$.
Given a MOMDP $M$, a policy $\pi$ and a state $s$, we define the random variable representing the cumulative reward attained in $H$ timesteps starting in $s$ and following $\pi$:
\begin{equation}
\label{eq:exp_cumul}
 \bm{X}^{\pi}(s) = \left( \sum_{t=1}^{H} \mathbf{r}_t\ \big|\ \pi, s_0=s \right), 
\end{equation}
\noindent where $\mathbf{r}_t$ is the reward observed at timestep $t$ according to $\mathbf{R}$.

\begin{defn} \label{def:value}
    The \textit{value} of a policy $\pi$ is a function $\mathbf{V}^\pi:S \rightarrow  \mathbb{R}^D$ such that:
    \begin{align}
        \mathbf{V}^\pi(s) = (V^\pi_1(s), ..., V^\pi_D(s))= \mathbb{E}[\bm{X}^\pi(s)].
    \end{align}
\end{defn}

 In the remainder of this paper, when $s = \bar{s}$ we will omit it and simply write $\bm{X}^{\pi}$ and $\mathbf{V}^\pi= (V^\pi_1, ..., V^\pi_D)$.
 For a set of policies $\Pi$, we define the set of points $\set{V}(\Pi) \subseteq \mathbb{R}^D$ such that  $\set{V}(\Pi) = \{ \mathbf{V}^\pi\ |\ \pi\in \Pi \}$.
Often there will be no single best policy that can be chosen, as one may encounter a situation where we have $V^{\pi}_i > V^{\pi'}_i$, but, $V^{\pi}_j < V^{\pi'}_j$ for some policies $\pi,\pi'$ and $i,j\in \{1,...,D\}$. However, we can define a partial ordering over multi-objective values:

\begin{defn} \label{def:pareto_dominate}
    We say that a point $\bm{u} \in \mathbb{R}^D$ \textit{weakly Pareto dominates} another point $\mathbf{v} \in \mathbb{R}^D$ (denoted $\mathbf{u} \succeq \mathbf{v}$), if $\forall i. u_i \geq v_i$;
    $\bm{u}$ \textit{Pareto dominates} $\mathbf{v}$ (denoted $\mathbf{u} \succ \mathbf{v}$), if $\mathbf{u}\succeq\mathbf{v} \land \mathbf{u} \neq \mathbf{v}$;
    $\bm{u}$ is \textit{incomparable} to $\mathbf{v}$ (denoted $\mathbf{u} \parallel \mathbf{v}$), if $\mathbf{u} \nsucceq \mathbf{v} \land \mathbf{v} \nsucceq \mathbf{u}$.         
\end{defn}


Given Pareto domination as partial ordering over multi-objective values, we can now define an `optimal set' of policies, in which no policy Pareto dominates another. 

\begin{defn} \label{def:pareto_front}
	Let $\Pi$ be a set of policies. The \textit{Pareto Front} $PF(\Pi)$ is the set of policies in $\Pi$ that are not Pareto dominated by any other policy in $\Pi$:
    \begin{align}
        PF(\Pi) &= \{\pi\in\Pi\ |\ \forall\pi'\in\Pi.\ \mathbf{V}^\pi \nsucc \mathbf{V}^{\pi'} \}.
    \end{align}

\end{defn}

An alternative method to overcome the lack of a total order for multi-objective values is to project them onto some scalar value that can be compared.
First, we define the set $\cl{W}_D \subset \mathbb{R}^D$ of~\emph{normalised $D$-dimensional} weights as:

\begin{equation}
\cl{W}_D = \{\mathbf{w}=(w_1,\hdots,w_D) \ |\ \forall i.\ w_i\geq 0 \land \sum_{i=1}^D w_i = 1\}.
\end{equation}

\begin{defn}
    A \textit{scalarisation function} $f(\mathbf{V};\mathbf{w})$ maps a multi-objective value into a scalar value, where $\mathbf{w}\in \cl{W}_D$. The \textit{linear scalarization} is $f_\ell(\mathbf{v};\mathbf{w})=\mathbf{w}^\top\mathbf{v}$, which is \textit{(strictly) monotonic} (i.e. $\mathbf{u} \succ \mathbf{v}$ if and only if $\forall \mathbf{w}\in \cl{W}_D.\ f(\mathbf{u};\mathbf{w}) > f(\mathbf{v};\mathbf{w})$).
\end{defn}

A weight $\mathbf{w}\in \cl{W}_D$ and the linear scalarisation function gives a total ordering over policies, as $\pi,\pi'$ can be compared by comparing the scalar values $\mathbf{w}^\top \mathbf{V}^\pi$ to $\mathbf{w}^\top \mathbf{V}^{\pi'}$. Therefore, we can define another optimal set of policies as those that are optimal for some weight vector:

\begin{defn} \label{def:ccs}
    The \textit{Convex Hull} $CH(\Pi)$ is the set of policies in $\Pi$ that are optimal for some linear scalarization. 
    \begin{align}
         CH(\Pi) &= \{ \pi\in\Pi\ |\  \exists\mathbf{w}. \forall\pi'\in\Pi.\ \mathbf{w}^\top\mathbf{V}^\pi \geq \mathbf{w}^\top\mathbf{V}^{\pi'} \}.
    \end{align}

\end{defn}

If $\Pi$ is defined as the set of stochastic policies then $PF(\Pi)=CH(\Pi)$.
In order to represent convex hulls and Pareto fronts of $\Pi$ more compactly, we consider the notion of \textit{Convex Coverage Set} of $\Pi$.

\begin{defn} \label{def:undominated}
    A \textit{Convex Coverage Set} $CCS(\Pi)$ is any set such that:
    \begin{align}
\set{V}(CCS(\Pi))=\set{V}(CH(\Pi)).
    \end{align}
\end{defn}

For computational reasons, one typically wants to maintain a minimal set of policies that is still a $CCS(\Pi)$.
We refer the reader to~\citep{roijers2013survey} for more details on the relation  between $PF(\Pi)$, $CH(\Pi)$ and $CCS(\Pi)$.
%
%

To compare sets of points it is common to consider the \textit{hypervolume}, due to its monotonicity with respect to Pareto domination \citep{vamplew2011empirical}.
\begin{defn} \label{def:hypervolume}
    The \textit{hypervolume} of a set of points $\set{P}\subseteq \mathbb{R}^D$ can be defined with respect to a reference point $\mathbf{o}$ as follows:
    \begin{align}
        HV(P,\mathbf{r}) = \lambda_D(\{\mathbf{v}\in\mathbb{R}^D \ |\ \exists \mathbf{p}\in \set{P}.\ \mathbf{p} \succeq \mathbf{v} \succeq \mathbf{o} \}), \label{eqn:hypervolume}
    \end{align}
    where $\lambda_D$ is the $D$-dimensional Lebesgue measure. For example, in two dimensions with $\mathbf{o} = (0,0)$, this equates to the area between $\set{P}$ and the 
    $x$ and $y$ axes.
\end{defn}

\subsection{Convex Hull Value Iteration} \label{sec:chvi}

\citet{barrett2008learning} proposed an algorithm that extends Value Iteration~\citep{Bellman:1957} to handle multiple objectives, named \textit{Convex Hull Value Iteration} (CHVI). CHVI  computes $\set{V}(PF(\Pi))$ for an infinite-horizon MOMDP by computing a finite set of deterministic policies $\Pi_d$ that is a $CCS(\Pi)$, where $\Pi$ is the set of stochastic policies.
 The same algorithm can be adapted to compute  a $CCS(\Pi)$ for a finite-horizon MOMDP, by computing the values for each timestep in the horizon. 

As the first step to arriving at CHVI, we define arithmetic rules over a set of points.
For sets of points $\set{P}, \ \set{P'} \subseteq \mathbb{R}^D$, point $\mathbf{b}\in\mathbb{R}^D$, and scalar $k\in\mathbb{R}$, we define:
\begin{align}
    \set{P} + \set{P'} &= \{ \mathbf{p}+\mathbf{p'}\ |\ \mathbf{p}\in\set{P}, \mathbf{p'}\in\set{P'}\}, \label{eq:ch_arith2} \\
    \mathbf{b} + k\set{P} &= \{ \mathbf{b} + k\mathbf{p}\ |\ \mathbf{p}\in\set{P} \}. \label{eq:ch_arith3}
\end{align}

Then, to arrive at the CHVI algorithm, we can replace the fixed point equations used in Value Iteration as follows:
\begin{align}
    \set{V}(s) &= \begin{cases}
        \set{0} & \text{ if } s \text{ terminal,} \\
        \mtt{prune } (\set{Q}(s,a)) & \text{ otherwise,}
    \end{cases} \label{eqn:chvi_v} \\
    \set{Q}(s,a) &= \mathbb{E}\left[\bm{R}(s,a) + \set{V}(s')\ |\ s,a \right], \label{eqn:chvi_q}
\end{align}

\noindent where $\set{V}(s)$ and $\set{Q}(s,a)$ are sets of points, $\set{0}=\{(0, \hdots, 0)  \} \subset \mathbb{R}^D$ and $\mtt{prune}$ is an operation that removes Pareto-dominated points from a set. For details on how to compute 
$\mtt{prune}$ see the survey by \citet{roijers2013survey}. The expectation is taken with respect to the next state $s'\sim T(s,a,\cdot)$, and can be computed using equations (\ref{eq:ch_arith2}) and (\ref{eq:ch_arith3}) in the standard definition of expectation for a discrete random variable.

\subsection{Contextual Regret} \label{sec:cmab}

The \textit{Contextual Bandit} framework extends the standard \textit{Multi-Armed Bandit} (MAB) problem~\citep{auer2002finite}. In this framework we are given a context $x_k\in X$ in round $k$ and have to pick an arm $y_k\in Y$.
Subsequently a scalar reward $r_k\sim U(x_k,y_k)$ is received, where  $U(x_k,y_k)$ is a reward distribution for arm $y_k$ in context $x_k$, with expected value $\mu(x_k,y_k)$. We denote the shared context-arm space as $\cl{P}\subseteq X\times Y$, and denote the optimal value of context $x\in X$ as $\mu^*(x)=\sup_{\{y\in Y \ | \ (x,y)\in\cl{P}\}} \mu(x,y)$.

\begin{defn} \label{def:contextual_regret}
    The \textit{contextual regret} for $N$ many rounds, is defined as follows:
    \begin{align}
        \mtt{reg}(N) = \sum_{k=1}^N \mu^*(x_k) - \mu(x_k,y_k),
    \end{align}
  \noindent where $y_k$ is the arm chosen at round $k$.
\end{defn}

Typically in contextual MABs, the aim is to maximise the \textit{total payoff} of $\sum_{k=1}^N r_k$, which corresponds to minimizing the contextual regret.
Algorithms usually aim to achieve \emph{sublinear regret}, i.e. pick arms such that $\mtt{reg}(N)\in o(N)$, where $f(x)\in o(x)$ if $\frac{f(x)}{x}\rightarrow 0$ as $x\rightarrow \infty$.

As the problem is very general, works often make additional assumptions, or assumptions of additional knowledge to make the problem more tractable. 
In particular, one assumption made by \citet{slivkins2014contextual} is having access to a metric space $(\cl{P},d)$ called the \textit{similarity space}, such that the following Lipschitz condition holds:
\begin{align}
    |\mu(x,y)-\mu(x',y')| \leq d((x,y),(x',y')). \label{eq:lipshitz}
\end{align}

This allows us to reason about the expected values of contexts that have not been seen in the past.

\section{Convex Hull Monte-Carlo Tree-Search} \label{sec:chmcts}

To motivate the design of our tree-search algorithm, we consider an example MOMDP with a horizon of length three. We will see that any Monte-Carlo tree-search algorithm that maintains a \textit{sample average} at each node, i.e. the average reward obtained from each action, will perform suboptimally.

\begin{ex} \label{ex:v}
    Consider state $s_0$ and actions $a_1$, $a_2$ and $a_3$ in the MOMDP in Fig. \ref{fig:ex1} with a horizon of 2.
    It is clearly optimal to take action $a_3$ in $s_0$ for both objectives (and mixes thereof), and the true Pareto front is $F^*=\{(0,6),(6,0)\}$.
    The  sample averages for $a_1$, $a_2$ and $a_3$ are $\bar{\bm{Q}}(s_0,a_1)=(0,4)$, $\bar{\bm{Q}}(s_0, a_2)=(4,0)$  and $\bar{\bm{Q}}(s_0,a_3)=(6\lambda, 6(1-\lambda))$, for some $\lambda\in[0,1]$  (note that choosing action $a_3$ from $s_0$ will have a return of either $(6,0)$ or $(0,6)$, and  $(6\lambda, 6(1-\lambda))$, $\lambda\in[0,1]$ represents a weighted average of the two possible returns).
If one used these sample averages, the estimated Pareto front at  $s_0$ for a given $\lambda$ would be $F(\lambda)=\texttt{prune}(\{(0,4), (4,0),(6\lambda, 6(1-\lambda))\}) \neq F^*$.
Therefore, if we try to use the sample averages set $F(\lambda)$ to extract a policy, then it may be suboptimal. For example, if $\lambda=\frac{1}{2}$ then $\bar{\bm{Q}}(s_0,a_3)=(3,3)$ and the policy $\pi$ extracted to maximise only the second objective would have $\pi(s_0)=a_1$, since $\bar{\bm{Q}}(s_0,a_1)=(0,4)$, i.e. the sample average for the second objective is higher for $a_1$ than it is for $a_3$.
\end{ex}

\begin{figure}
    \centering
    \includegraphics[scale=0.35]{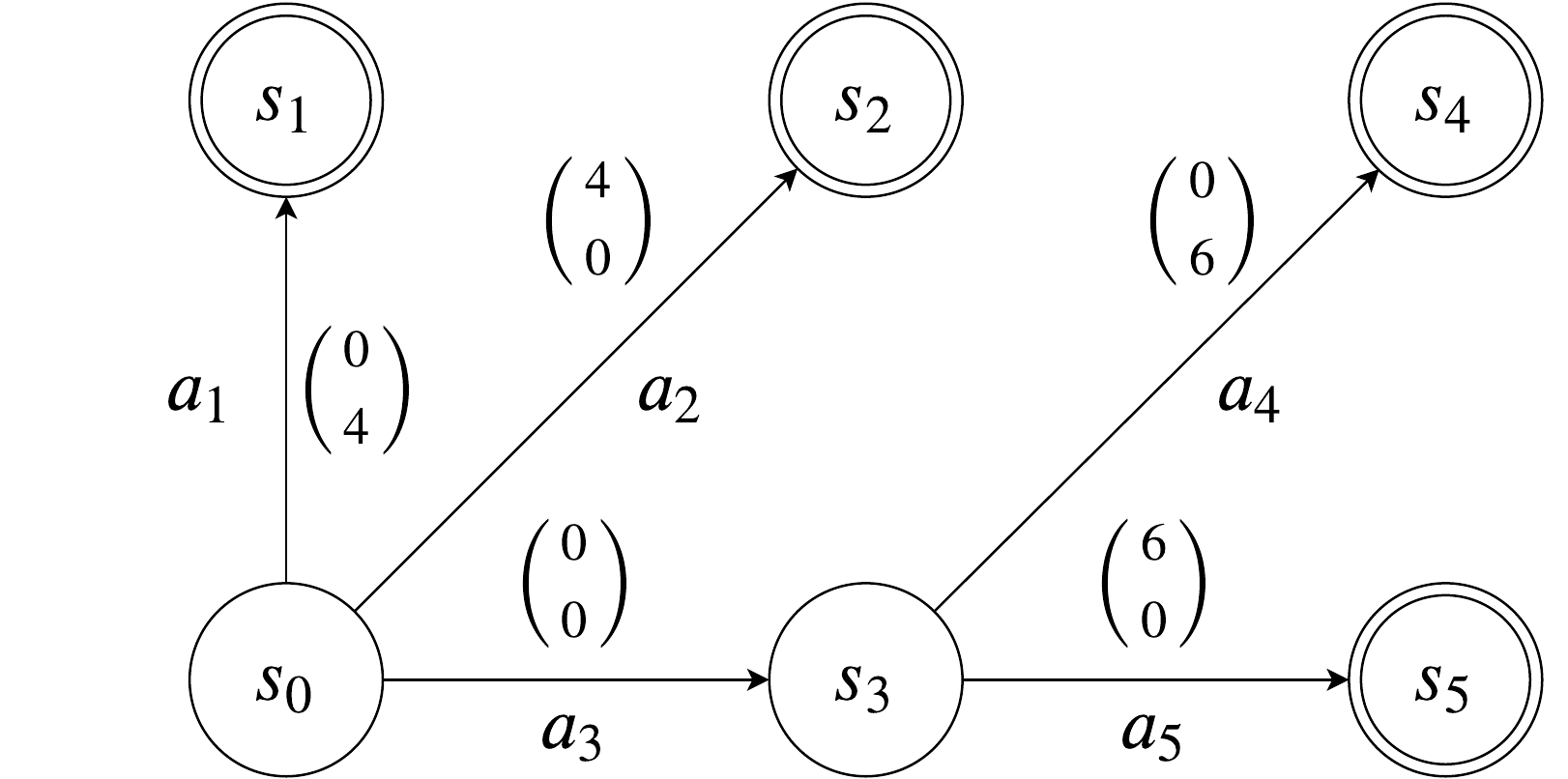}
    \caption{A deterministic MOMDP with six states, two-dimensional rewards, and initial state $s_0$. As it is deterministic, all transition probabilities are one, and so are omitted.} 
    \label{fig:ex1}
\end{figure}

Given Example \ref{ex:v}, it is clear that we need something more than just a single point in each node, so we will consider sets of points to approximate a Pareto front at each node. In the remainder of this section, we describe the backup functions that can be used as part of the THTS schema~\citep{keller2013trial}. Any algorithm that makes use of these backups in a Monte-Carlo Tree-Search we will refer to as a \textit{Convex Hull Monte-Carlo Tree-Search} (CHMCTS).

\subsection{Backup Functions}

For every THTS decision node, corresponding to some $s \in S$, we store a set of points  approximating $\set{V}(s)$, and for every chance node, corresponding to some $(s,a) \in S\times A$, we store a set of points approximating $\set{Q}(s,a)$.
The \texttt{backupDecisionNode} function updates the set approximating $\set{V}(s)$ using Equation (\ref{eqn:chvi_v}), but replacing each $\set{Q}(s,a)$ with its  approximation stored in the  corresponding child chance node.
 Similarly, \texttt{backupChanceNode} updates the approximation of $\set{Q}(s,a)$ using Equation~(\ref{eqn:chvi_q}), replacing each $\set{V}(s')$ with the approximation stored in the corresponding child decision node.

\section{Action Selection} \label{sec:action_selection}

Now we consider how to select actions from decision nodes (recalling decision nodes correspond to states in an MOMDP). We present the problem of policy selection (i.e. selecting all of the actions for a trial) framed as a Contextual Multi-Armed Bandit problem (Section \ref{sec:cmab}):

\begin{defn} \label{def:lcps}
    \textit{(Linear) Contextual Policy Selection} problem, is a special case of the Contextual Bandit problem. Let $M$ be a MOMDP and $\Pi$ the corresponding set of policies. For $N$ rounds (or \textit{trials}), we perform the follow sequence of operations: (1) receive a context $\bm{w}_k\in \cl{W}_D$; (2) select a policy $\pi_k$ to follow for this trial; (3) receive a cumulative reward $\bm{x}^{\pi_k}\in\mathbb{R}^D$, where $\bm{x}^{\pi_k} \sim \bm{X}^{\pi_k}$ as defined in Equation~(\ref{eq:exp_cumul}) -- recall that $\bm{V}^{\pi_k}=\mathbb{E}[\bm{X}^{\pi_k}]$. The objective over $N$ rounds is to select a sequence of policies $\{\pi_k\}_{k=1}^N$ that maximise the \textit{expected cumulative payoff} $\sum_{k=1}^N \bm{w}_k^\top \bm{V}^{\pi_k}$.
\end{defn}

\subsection{Design by Regret} \label{sec:design_by_regret}

In THTS, designing action selection using regret metrics typically has two main benefits  it is a direct measure of the \textit{online} performance of the action selection; and it is  a good way to balance the exploration-exploitation trade-off, as it selects arms proportionally to how good the performance of each arm has been in the past. Previous multi-objective works~\citep{drugan2013designing,chen2019pareto} have considered the notion of Pareto regret:

\begin{defn} \label{def:pareto_regret}
    The \textit{Pareto Suboptimality Gap} (PSG) for selecting policy $\pi$ is defined as:
    \begin{align}
        \Delta(\pi) = \inf\{\epsilon\in\mathbb{R}\ |\ \forall\pi'\in\Pi.\ \bm{V}^{\pi'} \nsucc \bm{V}^\pi + \epsilon \bm{1} \},
    \end{align}
    where $\bm{1}$ is a vector of ones. Intuitively, we can think of $\Delta(\pi)$ as how much needs to be added to $\bm{V}^\pi$ so that it is \textit{Pareto optimal}, i.e. it is not dominated by any other policy. Let $N_{\pi}$ be the number of times that $\pi$ was selected for the $N$ trials.
    The \textit{Pareto regret}  is defined as:
    \begin{align}
        \mtt{reg}_P(N) = \sum_{\pi\in\Pi} N_\pi\Delta(\pi).
    \end{align}
\end{defn}

To demonstrate why Pareto regret is not the most suitable regret metric, we consider Example \ref{ex:two}, which demonstrates that optimising for the Pareto regret does not correspond to our objective of computing the CCS for a MOMDP.  
%

\begin{ex} \label{ex:two}
    Consider the algorithm that uses the (single-objective) UCT algorithm~\cite{kocsis2006bandit} to find $\pi_i^* = \argmax_\pi \bm{V}^\pi_i$, the optimal policy for the $i$th objective and then continues to follow $\pi_i^*$.
   UCT is known to have sublinear regret for the $i$th objective and, because  $\pi_i^*$ is Pareto-optimal, it also has sublinear Pareto regret.
However, this algorithm does not align well with the objective of computing a CCS, as it focuses only on one Pareto-optimal policy and does not explore the rest of the CCS.
\end{ex}

Following from this, we introduce \textit{Linear Contextual Regret}, a special case of \textit{contextual regret}, that is well correlated with approximating the CCS.

\begin{defn}
    Let $\{\bm{w}_k\}_{k=1}^N$ be a \emph{context weight vector}, i.e. a sequence of weights  sampled uniformly from $\cl{W}_D$.
    The \textit{Linear Contextual Regret} (LCR) for the policy selection $\{\pi_k\}_{k=1}^N$ is defined by:
    \begin{align}
        \mtt{reg}_C(N, \{\bm{w}_k\}_{k=1}^N) = \sum_{k=1}^N \max_{\pi'\in\Pi} \bm{w}_k^\top \bm{V}^{\pi'} - \bm{w}_k^\top \bm{V}^{\pi}, 
    \end{align}
\end{defn}


We consider LCR because it will penalise any algorithm that cannot find a $CCS(\Pi)$: if  there exists some $\mathbf{V}^\pi  \in  \set{V}(PF(\Pi))$ such that no policy $\pi'$ with $\bf{V}^{\pi'} =\bf{V}^\pi$ was found, then  a weight close to $\bm{w}^* = \argmax_{\bm{w}} \bm{w}^\top \bm{V}^\pi$ is sampled, it will accumulate regret.

Recall from Definition~\ref{def:lcps}, we aim to maximise expected cumulative payoff, which, is equivalent to minimizing the expected LCR. Note that for any algorithm that achieves a sublinear LCR, the average regret will tend to zero, and in the limit of the number of trials the algorithm must almost surely act optimally for all weight vectors.

\subsection{Exploration Policies} \label{sec:exploration_policy}

We now formally define exploration policies:
\begin{defn} \label{def:expl_policy}
    An \textit{exploration policy} for MOMDP $M$ is a function of the form $\sigma:\cl{H}\times \cl{W}_D \rightarrow A$. 
\end{defn}

In essence, an exploration policy maps a history and a context weight vector to an action, i.e. it extends a policy to consider the context. 
Algorithms for the contextual policy selection problem must specify a sequence of exploration policies $\{\sigma_k\}_{k=1}^N$, where the MOMDP policy followed on trial $k$ will be $\pi_k=\sigma(\cdot;\bm{w}_k)$. The set of exploration policies $\Pi_e$ for a MOMDP can be divided into two broad classes:

\begin{defn}
    A \textit{context-free} exploration policy $\sigma$ is one such that $\forall \bm{u},\bm{w}\in \cl{W}_D. \sigma(\cdot;\bm{u}) = \sigma(\cdot;\bm{w})$.  An exploration policy is \textit{context-aware} if it is not context-free.
\end{defn}

Theorem~\ref{thrm:one} shows any sequence of context-free exploration policies can suffer linear LCR.

\begin{thrm}\label{thrm:one}
    For some MOMDP $M$, for any sequence of context-free policies $\{\sigma_k\}_{k=1}^N$, the expected LCR over $N$ trials is $\Omega(N)$. 
\end{thrm}

\begin{proofoutline}
    \begin{figure}
        \centering
        \includegraphics[scale=0.4]{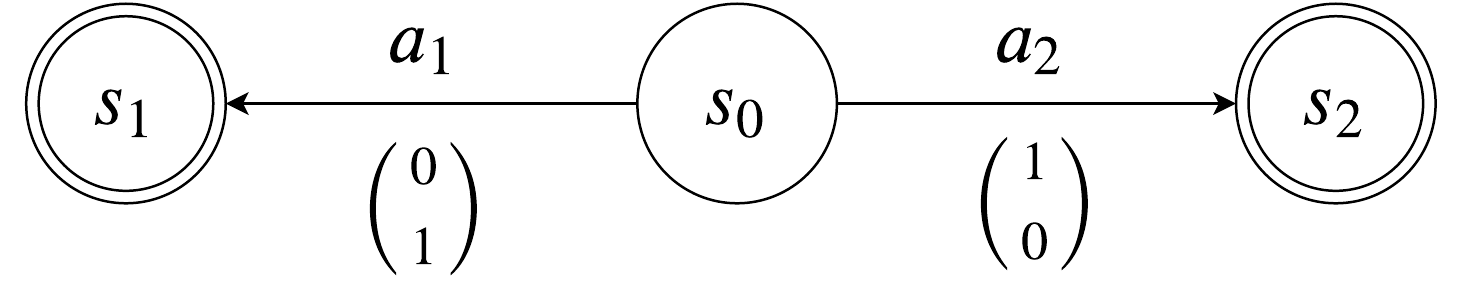}
        \caption{A deterministic MOMDP with three states, two-dimensional rewards and a horizon of one.} 
        \label{fig:proof_momdp1}
    \end{figure}
    
    (See Appendix B for full proof.) Consider the MOMDP from Fig. \ref{fig:proof_momdp1}. If we follow a context-free exploration policy $\sigma_k$ on the $k$th trial, then the action selected at $s_0$ is independent of the weight context vector. Assume, wlog, that $\sigma_t(s_0;\bm{w})=a_1$ for all $\bm{w}$. As $\bm{w}$ is sampled uniformly from $\cl{W}_D$, $a_1$ will be the suboptimal action with probability $\frac{1}{2}$. If $a_1$ is suboptimal, the expected regret suffered is $\frac{1}{2}$, because if the weight vector is varied from $(\frac{1}{2},\frac{1}{2})$ to $(1,0)$, then the contextual regret suffered varies from $0$ to $1$. On average, the context-free policies will suffer and expected regret of $\frac{1}{4}$ per trial, and thus the cumulative regret over $N$ is $\Omega(N)$.
\end{proofoutline}

\subsection{Context-Aware Action Selection}

\begin{figure}
    \centering
    \includegraphics[scale=0.3]{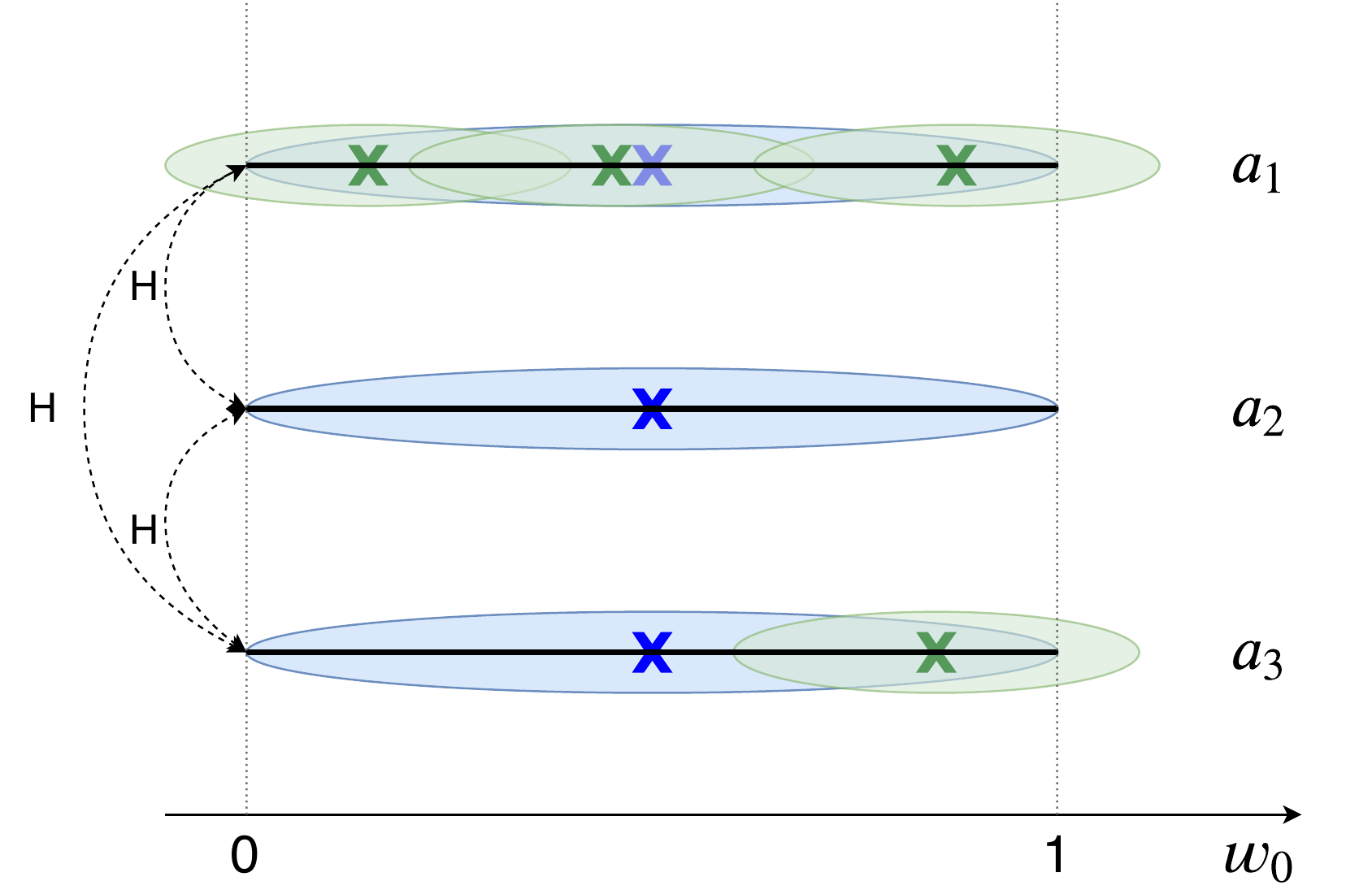}
    \caption{Visualisation of a snapshot of the CZ algorithm, running over three actions with two objectives, where $w_0$ is the context weighting for one objective.}
    \label{fig:cz}
\end{figure}

Extending UCB1 to handle a contextual MAB   can be hard, as  we have an uncountably infinite set of contexts, the weight vectors $\cl{W}_D$. If one has knowledge of a metric that satisfies  Equation~(\ref{eq:lipshitz}), the problem can be made more tractable. The metric allows contexts to be grouped (in sets with a fixed radius, i.e. balls), and maintain an average value over all contexts in the ball. Intuitively, smaller balls allow each context vector to have a more accurate value estimate maintained. These ideas underline the \textit{Contextual Zooming} (CZ) algorithm~\citep{slivkins2014contextual}, that modifies UCB1 to run over balls of contexts, and introduces balls of smaller radii as required. 

We present CZ  by defining the \textit{similarity space} (i.e. the metric over the context-arm space be $\cl{P}=\cl{W}_D\times \Pi$ satisfying Equation~(\ref{eq:lipshitz})); presenting CZ for \textit{policy selection}; and then adapting CZ to be used in  \texttt{selectAction} from THTS. A snapshot of CZ is visualised in Fig.~\ref{fig:cz}, with three actions. The initial three balls of radii one are blue, and balls of radii $\frac{1}{2}$ are green. $a_1$ has been covered by green balls, and will only use value estimates from these green balls, whereas $a_3$ will use the value of the one green ball that it has, \textit{if it is relevant}, and otherwise use the blue ball.

\textbf{Similarity Space.} Consider the linear contextual policy selection problem and let $d: \cl{W}_D\times \Pi \rightarrow \mathbb{R}_{\geq0}$ such that:
\begin{align}
    \d((\bm{w},\pi),(\bm{w}',\pi')) = \begin{cases}
        C(\pi) \Vert \bm{w} - \bm{w}' \Vert_\infty & \text{if } \pi = \pi', \\
        U & \text{otherwise}, 
    \end{cases}
\end{align}

\noindent where $U>0$ and $2U \geq C(\pi)\geq \Vert \bm{V}^{\pi} \Vert_{\infty}\geq 0$, i.e. $C(\pi)$ is a value that overestimates the infinity norm of the expected vector reward.
Furthermore, $U$ is an upper bound on the maximum scalarised reward that can be achieved in a single trial, i.e. $U \geq\sup_{\bm{w}\in \cl{W}_D,\pi\in\Pi} \bm{w}^\top \bm{V}^{\pi}$. For example, if $\bm{R}(s,a)\in [0,1]^D$ then these values can be set to $C(\pi)=U=H$. Additionally, observe the Lipshitz property (Equation (\ref{eq:lipshitz})) holds when $\pi\neq\pi'$ because $|\mu(\bm{w},\pi)-\mu(\bm{w}',\pi)|\leq U$ and it holds when $\pi=\pi'$ because:
\begin{align}
    |\mu(\bm{w},\pi)-\mu(\bm{w}',\pi)| 
        = & \Vert(\bm{w}^\top -\bm{w}'^\top) \bm{V}^{\pi}\Vert_\infty \\
        \leq & \Vert\bm{V}^{\pi}\Vert_\infty \Vert(\bm{w}^\top -\bm{w}'^\top) \Vert_\infty \\
        \leq & C(\pi) \Vert(\bm{w}^\top -\bm{w}'^\top) \Vert_\infty \\
        = & d((\bm{w},\pi),(\bm{w}',\pi)),
\end{align}

\noindent where in the first line we use the definition of $\mu$ and the fact that the modulus and infinity-norm operations are identical on scalar values. We also used the result that for any matrices $\Vert AB \Vert_\infty \leq \Vert A \Vert_\infty \Vert B\Vert_\infty$.

\textbf{Contextual Zooming.} CZ is an algorithm that achieves a contextual regret of $O(N^{1-1/(2+c)}\log N)$, where $c$ is the \textit{covering dimension}. The \textit{covering dimension} is related to how many balls of a fixed radius are needed to cover the similarity space. For a full explanation and derivation of the regret bound we refer the reader to~\citep{slivkins2014contextual}.

Throughout the algorithm a finite set of balls called the et of \textit{active balls} is maintained. Let $\cl{A}_k$ be the set of active balls at the start of trial $k$. For us, each ball corresponds to a set of context vectors and has an associated arm (i.e. policy). Whenever CZ needs to select an arm, it will find a set of \textit{relevant balls} in $\cl{A}_k$, compute an \textit{upper confidence bound} for each relevant ball and select the arm associated with the largest bound. The ball $B$ is \textit{relevant} for the weight vector $\bm{w}\in \cl{W}_D$ if there is some arm $\pi\in\Pi$ such that $(\bm{w},\pi)\in \mtt{dom}_k(B)$, where:
\begin{align}
    \mtt{dom}_k(B) = B - \left( \bigcup_{\{B'\in\cl{A}_k \ | \  r(B') < r(B)\}} B' \right),
\end{align}

\noindent with $r(B)$ the radius of $B$. $\mtt{dom}_k(B)$ is used to decide which balls are relevant to consider when making a choice for some context vector. The upper confidence bound $I_k(B)$ for each $B$ in the set of relevant balls during round $k$ is defined using the following equations:

\begin{align}
    \mtt{conf}_k(B) =& 4\sqrt{\frac{\log k}{1+n_k(B)}},  \label{eq:conf_radius} \\
    I^{\text{pre}}_k(B) =& \nu_k(B) + r(B) + \mtt{conf}_k(B), \\
    I_k(B)=&r(B)+\min_{B'\in\cl{A}_k} \left( I^{\text{pre}}_k(B')+\cl{D}(B,B') \right),
\end{align}

where $n_k(B)$ is the number of times that ball $B$ has been selected in the previous $k-1$ rounds, $\nu_k(B)$ is the average (scalarised) reward for ball $B$ in the previous $k-1$ rounds, $r(B)$ is the radius of the ball and $\cl{D}(B,B')$ is the distance between the centers of balls $B$ and $B'$. 


    




The algorithm then proceeds by repetitively applying the following two rules. 
\textit{Selection rule:} On round $k$ select the ball $B$, from those that are relevant, that has the maximum index $I_k(B)$. From that ball, select an arm arbitrarily such that $(\bm{w}_k,\pi)\in B$. 
\textit{Activation Rule.} If the ball $B$ that was selected satisfies $\mtt{conf}_k(B) \leq r(B)$ after this round, then a new ball is added to the active set, otherwise we set $\cl{A}_{k+1}=\cl{A}_k$. When adding a new ball, if $\pi$ was the arm selected by the selection rule in round $k$, then a new ball $B_{\text{new}}$ with center $(\bm{w}_k,\pi)$ is introduced with radius $\frac{1}{2}r(B)$, and we set $\cl{A}_{k+1}=\cl{A}_k \cup \{ B_{\text{new}} \}$.
To initialise the set of active balls we add one ball per arm $\pi$, with center $(\bar{\bm{w}},\pi)$ and radius $C(\pi)$ to $\cl{A}_0$, where $\bar{\bm{w}}=\frac{1}{d}\bm{1}$. 


\textbf{Contextual Zooming for Trees.}  Running CZ directly over policies is infeasible for MOMDPs, because the number of policies grows exponentially in the size of the state space, and additionally the distance metric does not allow two different policies to be close in the similarity space. Instead, we run CZ for the \texttt{selectAction} method in THTS at every decision node. So now we consider a new similarity space $(\cl{P}_s,d_s)$ associated with the state $s$ of the decision node, where $\cl{P}_s = \cl{W}_D \times A$, and $d_s$ is defined as:
\begin{align}
    d_s((\bm{w},a), (\bm{w}',a')) = \begin{cases}
        C_s(a) \Vert \bm{w} - \bm{w}' \Vert_\infty & \text{if } a = a', \\
        U & \text{otherwise}, 
    \end{cases}
\end{align}

\noindent where $C_s(a)=\sup_{\{\pi \ | \ \pi(s)=a\}} C(\pi)$. Given this similarity space, the CZ algorithm operates as before, using actions for the arms instead of policies in the contextual MAB problem. Note that for each decision node we, in fact, have a \textit{non-stationary} contextual multi-armed bandits problem, similar to~\citep{kocsis2006improved}. We refer to this action selection method as \textit{Contextual Zooming for Trees} (CZT), and when we use CZT for action selection in CHMCTS we call the algorithm \textit{Zooming CHMCTS}.

\section{Experiments}

To validate Zooming CHMCTS we will evaluate its performance on a variable-sized grid world problem, the \textit{Generalised Deep Sea Treasure} (GDST) problem~\citep{vamplew2011empirical,vamplew2017morl}. Moreover, we will consider its performance in both online and offline planning, where in online planning we assume that the agent follows each policy selected, and we want to maximise the cumulative payoff over many trials. In contrast, in offline planning, we do not care about the performance during the planning phase, but only the quality of policies that can be extracted afterwards.

\subsection{Experimental Setup}

In the GDST($c$,$p$) problem we consider a two-dimensional grid world, consisting of $c$ columns and a transition noise of $p\in[0,1]$. The submarine can move left, right, up or down each time step, with the submarine remaining stationary if it would otherwise leave the grid world. The submarine starts in the top left corner on each trial. On every timestep, transition noise $p$ indicates the probability that the submarine is instead swept by a current, moving it in a random direction. The seafloor becomes increasingly deep from left to right (with depth increasing by a small amount between zero and three inclusive for each column), but also holds increasing amounts of treasure at greater depths, ranging in values from one to 1000. There are two objectives, one is to collect the maximal amount of treasure, while the other is to minimise the time cost of reaching the treasure, where a cost of one is incurred for each timestep. Each trial concludes after either $100c$ steps or as soon as the submarine arrives at a piece of treasure. A visualisation of an environment for $c=5$ is given in Fig. \ref{fig:gdst}. During planning, we normalise rewards to the range $[0,1]$ to give equal priority over the different objectives.  

\begin{figure}
    \centering
    \includegraphics[scale=0.45]{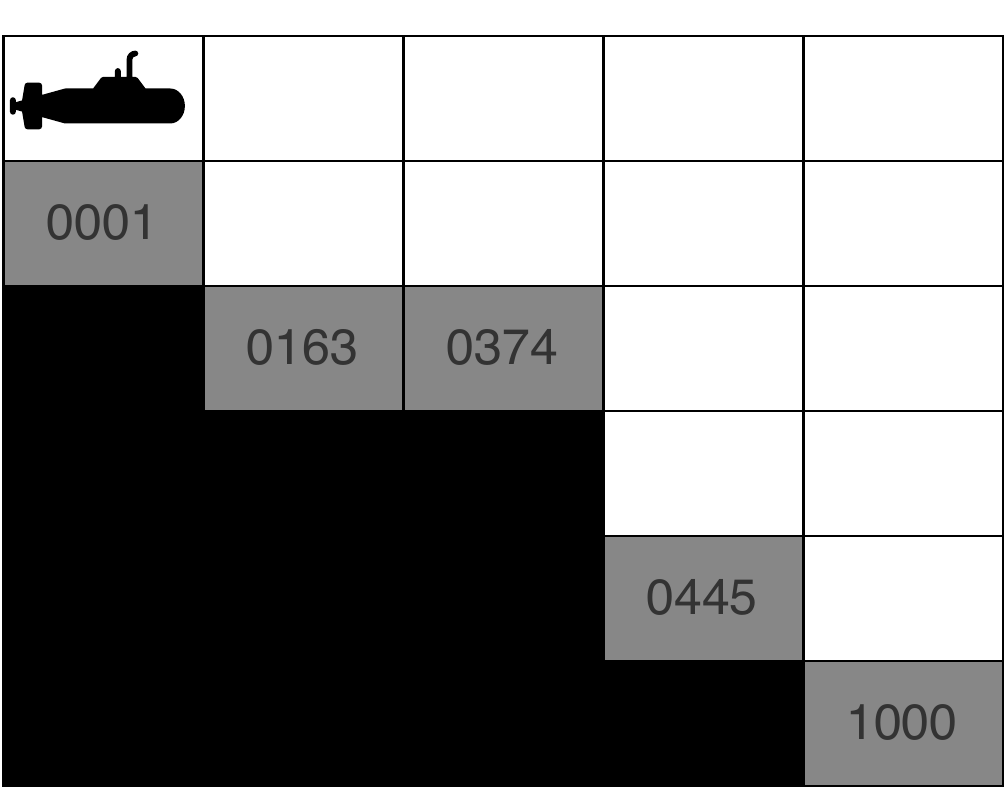}
    \caption{A visualisation of a GDST(5,p) problem, gray squares represent the treasure that can be obtained.}
    \label{fig:gdst}
\end{figure}

\subsection{Practical Considerations} \label{sec:optimizations}

Because the CHVI backups from Equations (\ref{eqn:chvi_v}) and (\ref{eqn:chvi_q}) can be computationally intensive, optimisations can be made to minimise the number of these backup operations performed. Where appropriate we use these following optimisations.

Firstly, \citet{lizotte2010efficient,lizotte2012linear} demonstrate that performing the relevant computations in \textit{weight space}, rather than \textit{value space}, is computationally more efficient. 
We additionally follow the \textit{backup optimisation} considered by \citet{perez2015multiobjective}.
We can identify when either a \texttt{backupChanceNode} or a \texttt{backupDecisionNode} does not change the value at a node, and subsequently prune the remaining backups for the trial. Also, it is possible to incorporate the ideas from UCD \citet{saffidine2012ucd} that allow nodes to be re-used when a state is re-visited.
Finally, for offline planning, we can use the idea of \textit{labelling} nodes, where a node is labelled if its value has converged. Nodes can be labelled when they are backed up and all of their children are labelled. Leaf nodes are always labelled. By using labelling, we can avoid searching parts of the tree that have already converged.

\subsection{Alternative Action Selection Methods}
We now state some other action selection methods used in the literature. Let $s$ be the state of the decision node that an action is being selected at, which has been visited $N(s)$ times. For any action $a$, let $N(s,a)$ be the number of times that the corresponding child node has been visited and let $\set{Q}(s,a)$ denote the Pareto front stored in the child node. Most prior works use an exploration policy of the form:
\begin{align}
    \nu(s;\bm{w}) = \argmax_{a}\ \zeta(\set{Q}(s,a); \bm{w}) + C\sqrt{\frac{\log(N(s))}{N(s,a)}}, 
\end{align}
where $C$ is some appropriate constant and $\zeta$ is some scalarization function that maps a Pareto front to a scalar value. \citet{perez2013online,perez2015multiobjective,perez2016multi} set 
\begin{align}
    \zeta(\set{Q}(s,a);\bm{w})=\frac{HV(\set{Q}(s,a),\mathbf{o})}{N(s)},
\end{align}

\noindent where $HV$ is the hypervolume function (Definition \ref{def:hypervolume}). This action selection is \textit{context-free} and we refer to the algorithm that results from this exploration policy as \textit{Hypervolume CHMCTS}. \citet{xu2017chebyshev} use the \textit{Chebychev scalarization} function instead, setting:
\begin{align}
    \zeta(\set{Q}(s,a);\bm{w})=\min_{\mathbf{p}\in \set{Q}(s,a)} \max\limits_i w_i |p_i - z_i|, \label{eq:chebychev_scalarization}
\end{align}

\noindent where $\bm{z}=(z_1,\hdots,z_D)$ is called the \textit{utopian point}, defined by $z_i = \max_\pi V^\pi_i$. When we use this action selection we call the algorithm \textit{Chebychev CHMCTS}. This action selection is actually \textit{context-aware}, and \citet{xu2017chebyshev} select the weight vector $\bm{w}$ randomly each trial. 

The ParetoUCB1 algorithm is used by \citet{chen2019pareto} for action selection in a MOMAB problem.
In our approach, we run ParetoUCB1 over the set $\bigcup_{a\in A} \set{Q}(s,a)$, rather than using vectors from child nodes as in point-based MOMCTS. When using this action selection, we refer to our algorithm as \textit{Pareto CHMCTS}.

\subsection{Regret Comparison} 

\begin{figure} 
    \centering
    \includegraphics[scale=0.4]{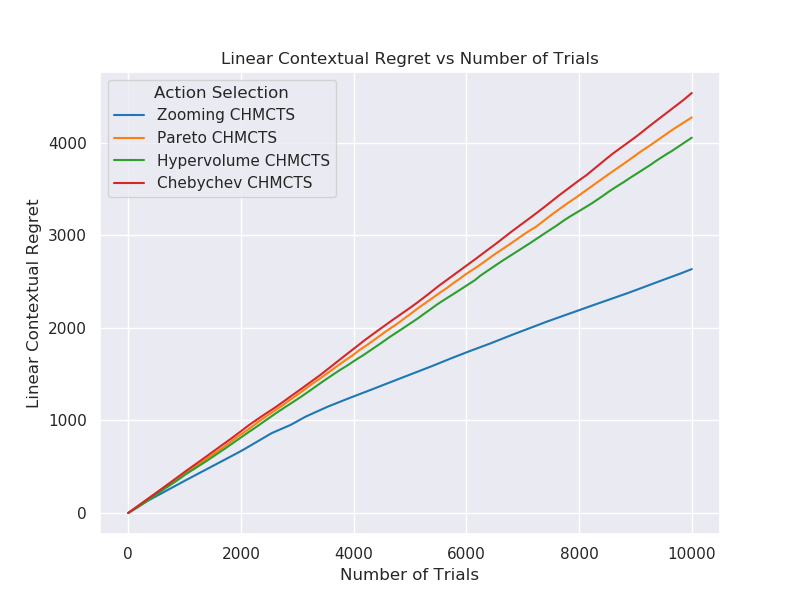}
    \caption{The LCR $\mtt{reg}_U(N)$ with respect to the number of trials $N$ of different CHMCTS algorithms. Each algorithm was run  five times. }
    \label{fig:regret_curves}
\end{figure}

In this section, we consider a GDST(7,0.01) instance, which is small enough for us to feasibly compute the true CCS, and thus the regret is computable. Fig. \ref{fig:regret_curves} shows a plot of the cumulative LCR over 100,000 trials, on the GDST(7,0.01) instance. As we can see in the figure, only Zooming CHMCTS manages to achieve a sublinear regret. These results are consistent with Theorem~\ref{thrm:one}, with context-free action selection methods suffering a linear regret. Interestingly, Chebychev CHMCTS suffers a linear regret too, despite being context-aware. From these results, we can hypothesise that a regret bound could be found for CZT. This demonstrates that Zooming CHMCTS outperforms all other variants for online planning (i.e. in the contextual policy selection problem).

\subsection{Comparing Action Selections}

\begin{figure}
    \centering
    \includegraphics[scale=0.4]{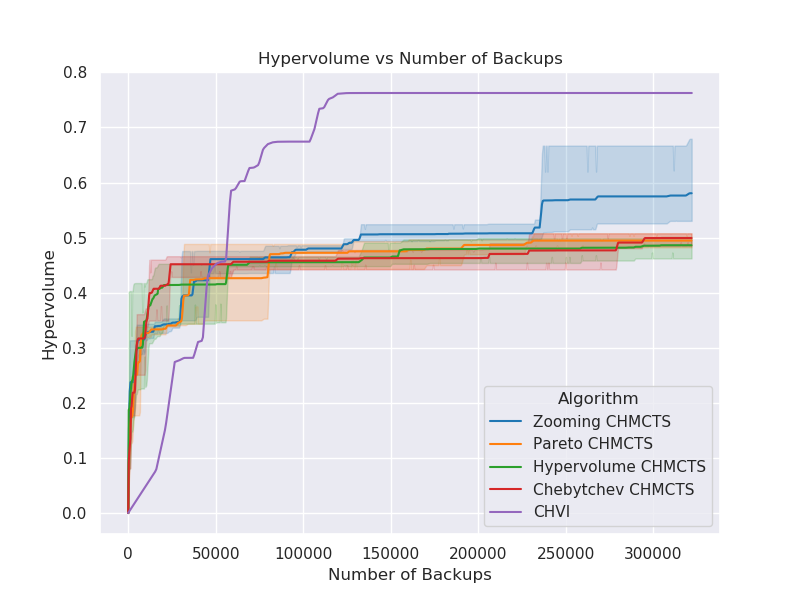}
    \caption{The hypervolume with respect to the number of CHVI operations. We compare each variant of CHMCTS with the baseline CHVI. 95\% Confidence intervals are plotted from 5 runs each.}
    \label{fig:hv_curves}
\end{figure}

To compare CHMCTS algorithms for offline planning, we compare the hypervolume of the CCS estimated in the root node, with respect to the number of backups required. Hypervolume is considered a suitable metric for the quality of a CCS~\citep{vamplew2011empirical}. We compare algorithms on a GDST(30,0.01) environment, using CHVI as a baseline in Fig.~\ref{fig:hv_curves}. 
We note that CHVI will compute the optimal CCS in $|S|\cdot H$ many backups if $S$ is the state space and $H$ the horizon. $|S|\cdot H$ is the least number of backups required to guarantee that the optimal value has been computed, so we expect CHVI to converge first. We see that Zooming CHMCTS continues to make steady progress towards the optimal hypervolume, while the other methods appear to quickly plateau. Also, we can see that CHMCTS algorithm would outperform CHVI if given a small computational budget. In the next section, we will also see that as the environment size is increased the performance of CHVI severely deteriorates.

\subsection{Scalability Analysis}

\begin{figure}
    \centering
    \includegraphics[scale=0.4]{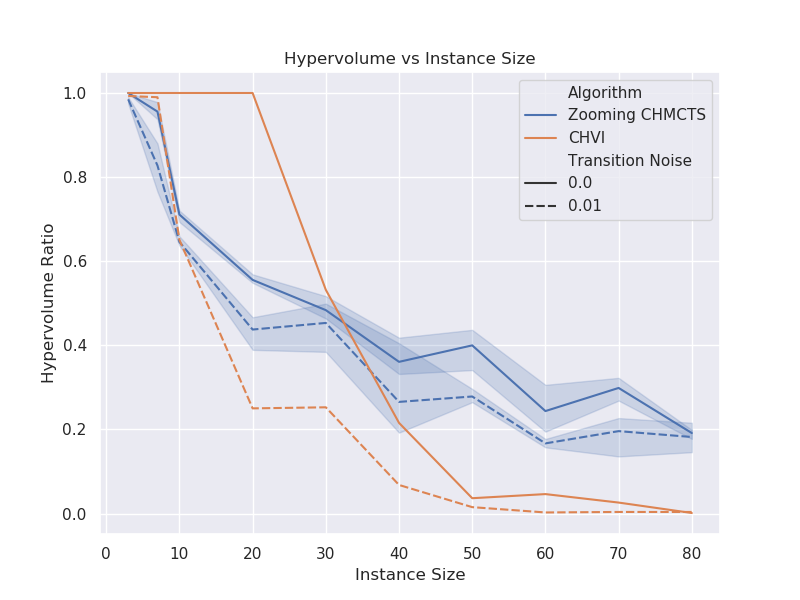}
    \caption{The approximate ratio of the hypervolume found, $\frac{\mtt{ehv(c,p)}}{\mtt{hv(c,0)}}$, with respect to problem size, i.e. $\mtt{c}$, with a budget of 25000 backups. 95\% Confidence intervals are plotted from 3 runs each.}
        \label{fig:scalability_curves}
\end{figure}

To understand how scalable CHMCTS is, we compare Zooming CHMCTS with CHVI on a range of different sized environments. In Fig.~\ref{fig:scalability_curves} we estimate how much of the true CCS is discovered given a budget of 25000 backups. We vary $c$ between three and 80, for both $p=0$ and $p=0.01$. 

Let $\mtt{hv(c,p)}$ denote the optimal hypervolume of a GDST(c,p) instance, and let $\mtt{ehv(c,p)}$ be the estimated hypervolume resulting from some algorithm on the same GDST(c,p) instance. The ratio of $\frac{\mtt{ehv(c,p)}}{\mtt{hv(c,p)}}$ indicates what proportion of the CCS was found by the algorithm. However, $\mtt{hv(c,p)}$ is infeasible to compute for $\mtt{p}>0$. Therefore we instead plot the lower bound $\frac{\mtt{ehv(c,p)}}{\mtt{hv(c,0)}} \leq \frac{\mtt{ehv(c,p)}}{\mtt{hv(c,p)}}$, which for GDST instances will always be in the range $[0,1)$.

In Fig.~\ref{fig:scalability_curves} we see that CHMCTS outperforms CHVI on instances with $c\geq 40$. We only compare Zooming CHMCTS because of computational constraints. Our results suggest that in the presence of a budget on the number of backups, there will be some threshold in the sizes of MOMDPs from which  CHMCTS  outperforms CHVI.

\section{Conclusion}

In this work, we posed planning in Multi-Objective Markov Decision Processes as a \textit{contextual multi-armed bandits problem}. We then discussed  why one should maintain approximations of the Pareto front in each tree node, to produce the Convex Hull Multi-Objective Tree-Search framework. By considering contextual regret, we introduced a novel action selection method and empirically verified it can achieve a sublinear regret and outperforms other state-of-the-art approaches.

Future work includes proving regret bounds of Multi-Objective Monte-Carlo Tree-Search algorithms and applying these algorithms in larger, real-world settings.

\section*{Acknowledgements}

This work is funded by Bossa Nova Robotics, Inc located at 610 22nd Street, Suite 250, San Francisco, CA 94107, USA,  UK Research and Innovation and EPSRC through the Robotics and Artificial Intelligence for Nuclear (RAIN) research hub [EP/R026084/1], and  the European Union Horizon 2020 research and innovation programme under grant agreement No 821988 (ADE).

\balance
\bibliography{painter}
\bibliographystyle{aaai}

\clearpage
\appendix

\setcounter{thrm}{0}
\setcounter{figure}{7}
\setcounter{equation}{28}

\section{Trial-Based Heuristic Tree Search} \label{app:thts}

\begin{figure*}
	\centering
	\includegraphics[width=\textwidth]{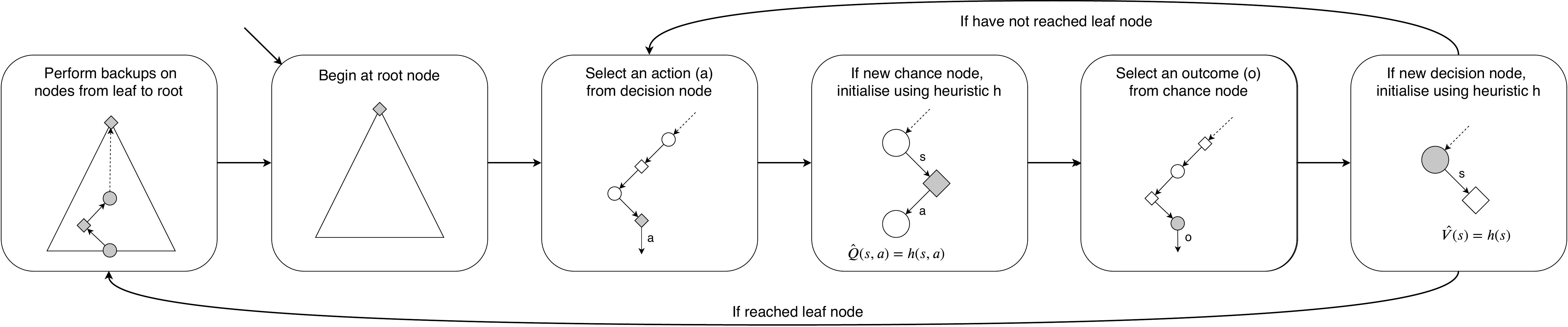}
	\caption{A visualization of the THTS schema.}
	\label{fig:thts}
\end{figure*}

In this appendix we give a brief overview of each of the subroutines used in THTS~\citep{keller2013trial} and what their purpose it, with psuedocode (Algorithm 1) and visualisation (Figure 8).

In Algorithm \ref{alg:thts} we give psuedocode for the THTS algorithm, which is split into three subroutines and presented as recursive functions for conciseness. In the loop starting on line \ref{line:outer_trials_loop}, we could easily add a function (say \texttt{generateContext}) that samples as context weight vector and passes it to the subroutines \texttt{thtsDecisionNode} and \texttt{thtsChanceNode}, to allow for contextual action selection.

The \texttt{initialiseNode} function provides an initial estimate of the value of a node, for example, this could just return zero for all nodes, or, this could incorporate a Monte-Carlo tree-search style rollout policies~\citep{chaslot2008monte}. 

The \texttt{visitDecisionNode} and \texttt{visitChanceNode} functions will update state stored in the chance or decision nodes. For example, in an implementation of UCT, the counts for how many times the node has been visited will be updated.

The \texttt{selectAction} and \texttt{selectOutcome} functions will select the next node to consider during a trial. So \texttt{selectAction} will be given a decision node $n_s$, and it will chose an action $a$, indicating $n_{sa}$ is the next node to visit. Similarly, \texttt{selectOutcome} will be given a chance node $n_{sa}$ and chose a successor state $s'$, indicating the next node to visit is $n_{s'}$. 

Finally, the \texttt{backupDecisionNode} and \texttt{backupChanceNode} functions correspond to updating the value estimate at each node from its successor nodes. For example, in a single-objective setting the value will be a scalar and could be updated using a \textit{Bellman backup} or \textit{Monte-Carlo backup}. In the multi-objective setting, our value associated with each node can be either a convex hull or Pareto front. 

\begin{algorithm}[tb]
    \caption{\texttt{THTS}}
    \label{alg:thts}
    \begin{algorithmic}[1]
        \Procedure{\texttt{THTS}}{MDP $M$, Timeout $T$}
            \State $n_0 \leftarrow \mtt{getRootNode}(M)$
            \While{$\mtt{notConverged}(n_0)$ and time $< T$} \label{line:outer_trials_loop}
                \State $\mtt{thtsDecisionNode}(n_0)$
            \EndWhile
        \EndProcedure \\
        
        \Procedure{\texttt{thtsDecisionNode}}{Node $n_d$}
            \If{$n_d$ uninintialised}
                \State $\mtt{initialiseNode}(n_d)$
            \EndIf
            \State $\mtt{visitDecisionNode}(n_d)$
            \State $N\leftarrow \mtt{selectAction}(n_d)$
            \For{$n_c$ in $N$}
                \State $\mtt{thtsChanceNode}(n_c)$
            \EndFor
            \State $\mtt{backupDecisionNode}(n_d)$
        \EndProcedure \\
        
        \Procedure{\texttt{thtsChanceNode}}{Node $n_c$}
            \If{$n_c$ uninintialised}
                \State $\mtt{initialiseNode}(n_c)$
            \EndIf
            \State $\mtt{visitChanceNode}(n_c)$
            \State $N\leftarrow \mtt{selectOutcome}(n_c)$
            \For{$n_d$ in $N$}
                \State $\mtt{thtsDecisionNode}(n_d)$
            \EndFor
            \State $\mtt{backupChanceNode}(n_c)$
        \EndProcedure 
    \end{algorithmic}
\end{algorithm}

\section{Proof of Theorem 1} \label{app:proof}

\begin{thrm}
    For any sequence of context-free policies $\{\sigma_i\}_{i=1}^N$, for some MOMDP $M$, the expected ULCR over $N$ trials is $\Omega(N)$. 
\end{thrm}

\begin{proof}

    Consider the MOMDP $M$ defined in Figure 9, with a horizon of one, and only two possible policies. Let $\{\bm{w}_k\}_{k=1}^N$ with $w_k=(\lambda_k,1-\lambda_k)$ sampled uniformly from $S_2=\{(\lambda,1-\lambda) | \lambda\in[0,1]\}$. Let $\pi_k=\sigma_k(\cdot;\bm{w}_k)$ be the policy selected on the $k$th trial, and let $\pi_k^*$ be the optimal policy for the $k$th trial. 
    
    Considering the two cases for what $\pi_k$ could be, we have: if $\pi_k(s_0)=a_1$, then $\sup_{\pi'} \bm{w}_k^\top (\bm{V}^{\pi'} - \bm{V}^{\pi_k}) = \max \{0,2\lambda_k-1\}$, and similarly, if $\pi_k(s_0)=a_2$, then $\sup_{\pi'} \bm{w}_k^\top (\bm{V}^{\pi'} - \bm{V}^{\pi_k}) = \max \{0,1-2\lambda_k\}$. Observing that 
    \begin{align}
        \int_0^1 \max \{0,2\lambda-1\} \d{\lambda} = \int_0^1 \max \{0,1-2\lambda\} \d{\lambda} = \frac{1}{4},
    \end{align}
    
    we must have $\mathbb{E}_{\bm{w}_k}[\sup_{\pi'} \bm{w}_k^\top (\bm{V}^{\pi'} - \bm{V}^{\pi_k})] = \frac{1}{4}$. Finally, summing over all terms in the definition of ULCR gives the result:
    \begin{align}
        \mtt{reg}_U(N) 
        =& \mathbb{E}_{\cl{W}}\left[ \sum_{k=1}^N \sup_{\pi'} \bm{w}_k^\top \bm{V}^{\pi'} - \bm{w}_k^\top \bm{V}^{\pi_k} \right] \\
        =& \sum_{k=1}^N \mathbb{E}_{\bm{w}_k} \left[ \sup_{\pi'} \bm{w}_k^\top \bm{V}^{\pi'} - \bm{w}_k^\top \bm{V}^{\pi_k} \right] \\
        =& \sum_{k=1}^N \frac{1}{4}  \\
        =& \Omega(N).
    \end{align}
\end{proof}

\begin{figure}
    \centering
    \includegraphics[scale=0.4]{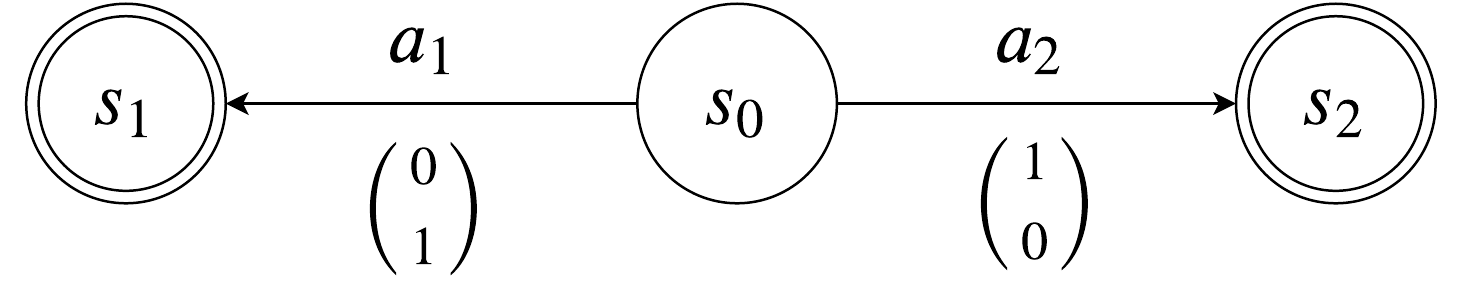}
    \caption{A deterministic MOMDP with three states, two-dimensional rewards and a horizon of one. Terminal states are marked with a double ring and the initial state is $s_0$. All actions are labelled with an action name, and the reward associated with the state-action pair. As all probabilities are one, probabilities are omitted from this figure.}
    \label{fig:proof_momdp2}
\end{figure}

\section{Proof that $\cl{D}$ is a Metric}

For completeness, we show in this section that the function $d:\cl{P}\times\cl{P}\rightarrow [0,\infty)$ is a metric on the similarity space $\cl{P}=S_D\times \Pi$. Recall that $D$ is defined in Equation (16) as follows:
\begin{align}
    d((\bm{w},\pi),(\bm{w}',\pi')) = \begin{cases}
        C(\pi) \Vert \bm{w} - \bm{w}' \Vert_\infty & \text{if } \pi = \pi', \\
        U & \text{otherwise}, 
    \end{cases} \tag{16}
\end{align}
where $U>0$ and $2U \geq C(\pi)\geq \Vert \bm{V}^{\pi} \Vert_{\infty}\geq 0$ for all $\pi\in\Pi$.

\begin{defn}
	The function $m:X\times X\rightarrow [0,\infty)$ is called a \textnormal{metric} on the set $X$ if the following properties hold for all $x,y,z\in X$:
	\begin{align}
		m(x,y) &\geq 0 					\tag{m1} \\
		m(x,y) &= 0 \iff x=y 			\tag{m2} \\
		m(x,y) &= m(y,x)					\tag{m3} \\
		m(x,y) &\leq m(x,z) + m(z,x). 	\tag{m4}
	\end{align}
\end{defn}

\begin{thrm}
	The function $d$ defined by Equation (16) is a metric on $\cl{P}=S_D\times \Pi$.
\end{thrm}

\begin{proof}
	Let $(\bm{w}_1,\pi_1),(\bm{w}_2,\pi_2),(\bm{w}_3,\pi_3)\in \cl{P}$. We will show that each of the properties (m1), (m2), (m3) and (m4) hold for $d$.
	
	\textbf{m1}: follows immediately from the assumptions $U>0$, $C(\pi)\geq 0$ and because $\Vert \bm{x} \Vert_{\infty} \geq 0$ for any $\bm{x}\in\mathbb{R}^D$.
	
	\textbf{m2}: it is logically equivalent to show that $(\bm{w}_1,\pi_1)=(\bm{w}_2,\pi_2) \implies d((\bm{w}_1,\pi_1),(\bm{w}_2,\pi_2))=0$~(m2a) and $(\bm{w}_1,\pi_1)\neq (\bm{w}_2,\pi_2) \implies d((\bm{w}_1,\pi_1),(\bm{w}_2,\pi_2))\neq 0$~(m2b) both hold. 
	
	\textbf{m2a}: if we are given that $(\bm{w}_1,\pi_1)=(\bm{w}_2,\pi_2)$ then:
	\begin{align}
		d((\bm{w}_1,\pi_1),(\bm{w}_2,\pi_2)) 
			&= C(\pi_1) \Vert \bm{w}_1 - \bm{w}_2 \Vert_\infty \\
			&= C(\pi_1) \Vert \bm{0} \Vert_\infty \\
			&= 0.
	\end{align}
	
	\textbf{m2b}: firstly, if $\pi_1 \neq \pi_2$, then we must have that $d((\bm{w}_1,\pi_1),(\bm{w}_2,\pi_2))=U>0$, because $U>0$ by assumption. Now consider the case when $\pi_1 = \pi_2$ and $\bm{w}_1\neq \bm{w}_2$, there must be some index $i$ such that $(\bm{w}_1)_i \neq (\bm{w}_2)_i$ and $|(\bm{w}_1)_i - (\bm{w}_2)_i| > 0$. Therefore, we must have that $\Vert \bm{w}_1 - \bm{w}_2 \Vert_\infty = \max_j |(\bm{w}_1)_j - (\bm{w}_2)_j| \geq |(\bm{w}_1)_i - (\bm{w}_2)_i| > 0$. Combining this with the assumption that $C(\pi)\geq 0$ for all $\pi\in\Pi$ gives the result.
	
	\textbf{m3}: if $\pi_1\neq \pi_2$, then immediately we have $d((\bm{w}_1,\pi_1),(\bm{w}_2,\pi_2)) = U = d((\bm{w}_2,\pi_2),(\bm{w}_1,\pi_1))$. When $\pi_1 = \pi_2$, m3 follows from the fact that $\Vert \bm{x} \Vert_{\infty} = \Vert -\bm{x} \Vert_{\infty}$ for any $\bm{x}\in\mathbb{R}^D$ and using $C(\pi_1)=C(\pi_2)$:
	\begin{align}
		d((\bm{w}_1,\pi_1),(\bm{w}_2,\pi_2)) 
			&= C(\pi_1) \Vert \bm{w}_1 - \bm{w}_2 \Vert_\infty \\
			&= C(\pi_1) \Vert \bm{w}_2 - \bm{w}_1 \Vert_\infty \\
			&= C(\pi_2) \Vert \bm{w}_2 - \bm{w}_1 \Vert_\infty \\
			&= d((\bm{w}_2,\pi_2),(\bm{w}_1,\pi_1)). 
	\end{align}
	
	\textbf{m4}: we consider three cases for this property, when $\pi_1 \neq \pi_2$~(m4a), $\pi_1 = \pi_2 \neq \pi_3$~(m4b) and $\pi_1 = \pi_2 = \pi_3$~(m4c). 
	
	\textbf{m4a}: in this case we must have one of $\pi_3\neq\pi_1$ or $\pi_3\neq\pi_2$, so assume that $\pi_3\neq\pi_1$ wlog. Given this, by the definition of $d$ (Equation (16)) and then by (m1), we have $d((\bm{w}_1,\pi_1),(\bm{w}_2,\pi_2)) = U = d((\bm{w}_1,\pi_1),(\bm{w}_3,\pi_3)) \leq d((\bm{w}_1,\pi_1),(\bm{w}_3,\pi_3)) + d((\bm{w}_3,\pi_3),(\bm{w}_2,\pi_2))$.
	
	\textbf{m4b}: one can show that we must have that $\Vert \bm{w}_1 - \bm{w}_2 \Vert_\infty \leq 1,$\footnote{Consider that for any $i$ the largest possible value of $\vert (\bm{w}_1)_i - (\bm{w}_2)_i\vert$ is one, because $\bm{w}_1,\bm{w}_2\in \cl{W}_D$.} then we must have that:
	\begin{align}
		d((\bm{w}_1,\pi_1),(\bm{w}_2,\pi_2)) 
			=& C(\pi_1) \Vert \bm{w}_1 - \bm{w}_2 \Vert_\infty \\ 
			\leq& C(\pi_1) \leq 2U = U + U \\
			=& d((\bm{w}_1,\pi_1),(\bm{w}_3,\pi_3)) \\ 
				&+ d((\bm{w}_3,\pi_3),(\bm{w}_2,\pi_2))
	\end{align}
	where the assumptions that $C(\pi_1)\leq 2U$, $\pi_1\neq \pi_2$ and $\pi_1\neq \pi_3$ are used. 
	
	\textbf{m4c}: finally, in this case we have that 
	\begin{align}
		d((\bm{w}_1,\pi_1),(\bm{w}_2,\pi_2)) 
			=& C(\pi_1)\Vert\bm{w}_1\!-\!\bm{w}_2 \Vert_\infty \\
			=& C(\pi_1) \Vert \bm{w}_1\!-\!\bm{w}_3\!+\!\bm{w}_3\!-\!\bm{w}_2\!\Vert_\infty \\
			\leq& C(\pi_1) \Vert \bm{w}_1\!-\!\bm{w}_3 \Vert_\infty \\
				&+ C(\pi_3) \Vert \bm{w}_3\!-\!\bm{w}_2 \Vert_\infty \\
			=&d((\bm{w}_1,\pi_1),(\bm{w}_3,\pi_3)) \\
				&+ d((\bm{w}_3,\pi_3),(\bm{w}_2,\pi_2))
	\end{align} 
	where we used the triangle property of norms and by (m4c) it must be the case that $C(\pi_1)=C(\pi_3)$.
\end{proof}

\end{document}